\documentclass[10.5pt,letterpaper]{article}

\usepackage{graphicx} %

\usepackage{amsmath,amssymb,amsthm}

\usepackage{color,mathtools}
\usepackage[round]{natbib}

\usepackage{multirow}
\usepackage{enumitem}

\usepackage{algorithm}
\usepackage{algorithmic}

\renewcommand{\cite}[1]{\citep{#1}}

\usepackage[hyperindex,breaklinks]{hyperref}

\usepackage{amssymb,amsmath,amsthm,dsfont}

\providecommand{\lin}[1]{\ensuremath{\left\langle #1 \right\rangle}}

\providecommand{\norm}[1]{\left\lVert#1\right\rVert}

  \providecommand{\R}{\mathbb{R}} 
  
  \providecommand{\E}{{\mathbb E}}
  \providecommand{\Eb}[1]{{\mathbb E}\left[#1\right] }       
  \providecommand{\EE}[2]{{\mathbb E}_{#1}\left.#2\right. }  

  \DeclareMathOperator*{\argmin}{arg\,min}

  \renewcommand{\gg}{\mathbf{g}}

  \providecommand{\xx}{\mathbf{x}}
  \providecommand{\yy}{\mathbf{y}}

  \providecommand{\cF}{\mathcal{F}}

  \providecommand{\cO}{\mathcal{O}}

\RequirePackage[colorinlistoftodos,bordercolor=orange,backgroundcolor=orange!20,linecolor=orange,textsize=scriptsize]{todonotes}
\providecommand{\comment}[2]{\todo[inline,caption={}]{\textbf{#1: }#2}}%
\providecommand{\inlinecomment}[3]{%
  {\color{#1}#2: #3}}%
\newcommand\commenter[2]%
{%
  \expandafter\newcommand\csname i#1\endcsname[1]{\inlinecomment{#2}{#1}{##1}}
  \expandafter\newcommand\csname #1\endcsname[1]{\comment{#1}{##1}}
}

\newtheorem{lemma}{Lemma}

\newtheorem{assumption}{Assumption}
\newtheorem{theorem}[lemma]{Theorem}

\definecolor{mydarkblue}{rgb}{0,0.08,0.45}

\usepackage{fullpage}

\commenter{seb}{red}

\title{Unified Optimal  Analysis of the (Stochastic) Gradient Method}
\author{Sebastian U. Stich\thanks{\texttt{sebastian.stich@epfl.ch}, Machine Learning and Optimization Lab (MLO), EPFL, Switzerland.}\\ EPFL \vspace*{-5mm}
}
\date{}

\begin{document}
\maketitle

\begin{abstract} 
In this note we give a simple proof for the convergence of stochastic gradient (SGD) methods on $\mu$-convex functions under a (milder than standard) $L$-smoothness assumption. 
We show that for carefully chosen stepsizes SGD
converges after $T$ iterations as $\cO\left( L \norm{\xx_0-\xx^\star}^2 \exp \bigl[-\frac{\mu}{4L}T\bigr] + \frac{\sigma^2}{\mu T} \right)$ where $\sigma^2$ measures the variance in the stochastic noise. For deterministic gradient descent (GD) and SGD in the interpolation setting we have $\sigma^2 =0$ and we recover the exponential convergence rate. The bound matches with the best known iteration complexity of GD and SGD, up to constants.
\end{abstract} 

\section{Introduction}
\label{sec:intro}
We consider the unconstrained optimization problem 
\begin{align*}
 f^\star := \min_{\xx \in \R^n} f(\xx)\,,
\end{align*}
where $f \colon \R^n \to \R$ is a convex continuously differentiable function.
We consider a stochastic approximation scenario---comprising the classic deterministic setting---where only unbiased estimates of the gradient of $f$ are available and study the convergence rate of stochastic gradient descent (SGD).

Formally, we assume that we have an increasing sequence of $\sigma$-fields $\{\cF_t\}_{t \geq 0}$, such that $\xx_0 \in \R^n$ is $\cF_0$ measurable and such that for all $t \geq 0$ the iterates of SGD are given as:
\begin{align}
 \xx_{t+1} = \xx_t - \gamma_t \gg_t\,, \label{eq:sgd}
\end{align}
where %
 $\{\gamma_t\}_{t \geq 0}$ denotes a sequence of (positive) stepsizes and $\gg_t \in \R^n$ is a stochastic gradient of $f$, satisfying the following three assumptions.
\begin{assumption}[Unbiased gradient oracles] \label{ass:1}
Almost surely,
\begin{align}
 \Eb{\gg_t \mid \cF_{t}} = \nabla f(\xx_t)\,, \qquad \forall t \geq 0\,, \label{def:unbiased}
\end{align}
where here $\nabla f(\xx_t)$ denotes the gradient of $f$ at $\xx_t$.
\end{assumption} 
\begin{assumption}[$(L,\sigma)$-smoothness] \label{ass:2}
 There exists two constants $L, \sigma^2 \geq 0$, s.t.
 \begin{align}
 \Eb { \norm{\gg_t}^2 \mid \cF_{t} } \leq 2L(f(\xx_t) - f^\star) + \sigma^2\,, \qquad \forall t \geq 0\,. \label{def:smooth}
 \end{align}
\end{assumption}
\noindent
This assumption generalizes the standard smoothness assumption as we will explain in Section~\ref{sec:examples} below. We further assume that $f$  is $\mu$-convex (with respect to the optimum $\xx^\star$---a slight relaxation of the standard assumption, see e.g.\ \cite{Necoara2019:linear}) and denote by $\xx^\star$ a minimizer of $f$ in $\R^n$.
\begin{assumption}[$\mu$-convexity] \label{ass:3}
There exists $\xx^\star \in \argmin_{\xx \in \R^n} f(\xx)$ and a constant $\mu \geq 0$ with
\begin{align}
 \frac{\mu}{2} \norm{\xx-\xx^\star}^2 + f(\xx)-f^\star \leq \lin{\nabla f(\xx),\xx-\xx^\star}\,, \qquad \forall \xx \in \R^n\,. \label{def:strong}
\end{align}
Note that this assumption is weaker than the standard strong convexity (for $\mu > 0$) or convexity (for $\mu=0$) assumptions that require such an inequality to hold for arbitrary pairs $\xx,\yy \in \R^d$ and not only for $\yy= \xx^\star$.
\end{assumption}

\subsection{Contribution} \label{sec:contribution}
Let $\{\xx_t\}_{t \geq 0}$ denote the iterates of~\eqref{eq:sgd}. We show that for appropriate stepsizes $\gamma_t$ and an appropriately defined average iterate after $T$ iterations, $\bar \xx_{T} := \frac{1}{W_T} \sum_{t=0}^T w_t \xx_{t}$ for weights $w_t \geq 0$ and $W_T:=\sum_{t=0}^T w_t$, it holds for $R= \norm{\xx_0 - \xx^\star}$:
\begin{align*}
 \E {f(\bar \xx_T)} -f^\star + \mu \E {\norm{\xx_{T+1}-\xx^\star}^2} = \cO \left( \min \left \{ L R^2 \exp\left[-\frac{\mu T}{4L}\right] + \frac{\sigma^2}{\mu T}\,, \frac{L R^2}{T} + \frac{\sigma R}{\sqrt{T}}\right\} \right)\,.
\end{align*}
We further also give a simpler proof that shows, up to polylogarithmic factors\footnote{
Here we follow the standard convention that $\cO$ hides constants and $\tilde \cO$ hides constants and factors polylogarithmic in the problem parameters. For the ease of exposition, we sometimes ignore absolute constants $\nu$ in the exponent when discussing results and write $\cO(\exp[-T])$ instead of $\cO(\exp[-\nu T])$. We diligently report these constants when stating new results.},
\begin{align*}
 \E {f(\bar \xx_T)} -f^\star + \mu \E {\norm{\xx_{T+1}-\xx^\star}^2} = \tilde \cO \left(L R^2 \exp\left[-\frac{\mu T}{2L}\right] + \frac{\sigma^2}{\mu T} \right)\,,
\end{align*}
and that only relies on constant stepsizes in~\eqref{eq:sgd}. 

This analysis unifies the analyses of gradient descent and SGD for smooth functions. In the deterministic case and in the iterpolation setting (where $\sigma^2 = 0$), we recover the exponential convergence rates of these algorithms (up to a factor 4 in the exponent) when $\mu > 0$. Furthermore, the result for convergence in function values is tight up to absolute (non-problem specific) constants~\cite{Nesterov2004:book}. Similarly, in the stochastic setting we recover the best known rates not only for the function values but also for the squared distance of the last iterate to the optimum~\cite{Nemirovski2009:last,Shamir2013:averaging}.

\subsection{Related Work}
\label{sec:related}

Whilst the first analyses of stochastic gradient descent (SGD)~\cite{Robbins:1951sgd} focused on asymptotic results~\cite{Chung1954:sgd}, the focus shifted to non-asymptotic results in recent years.

For $\mu > 0$, \citet{Moulines2011:nonasymptotic} give a bound $\E{\norm{\xx_T-\xx^\star}^2} = \tilde \cO \bigl((\tfrac{L}{\mu})^2 R^2 \exp\bigl[-\tfrac{\mu}{L}T\bigr] + \frac{\sigma^2}{\mu^2 T} \bigr)$, this was later improved by~\citet{Needell2016:sgd} to $\E{\norm{\xx_T-\xx^\star}^2} = \tilde \cO \bigl(\tfrac{L}{\mu} R^2 \exp\bigl[-\tfrac{\mu}{L} T\bigr] + \frac{\sigma^2}{\mu^2 T} \bigr)$.  Up to polylogarithmic factors this is the same rate as we show here in a slightly more general setting---however, their result only covers the distance $\norm{\xx_T-\xx^\star}^2$ of the iterates. Deducing from this result a rate for the function values via the smoothness inequality $f(\xx_T)-f^\star \leq \tfrac{L}{2}\norm{\xx_T-\xx^\star}^2$ introduces a superflous condition number factor $\frac{L}{\mu}$.

In the quest of deriving optimal rates---up to constant factors---in function suboptimality, different averaging schemes have been studied~\cite{Ruppert1988:averaging,Polyak1990:averaging,Rakhlin2012suffix,Shamir2013:averaging}. \citet{Lacoste2012:simpler} give a simple proof for $f(\bar \xx_T)-f^\star = \cO\bigl(\frac{G^2}{\mu T} \bigr)$, where here $G^2 \geq \sigma^2$ is an upper bound on the gradient norms, $\E{\bigl[\norm{\gg_t}^2\mid \cF_t \bigr]} \leq G^2$. Analyses under this assumption are not optimal in the deterministic setting where $\sigma^2 = 0$, but $G^2 > 0$ in general.
 
The $(L,\sigma^2)$-smoothness assumption appeared in this form recently in e.g.~\cite{Grimmer2019:sgd}, though very similar conditions have been studied in the literature~\cite{Bertsekas1996growth,Schmidt2013:fastconvergence,Needell2016:sgd,Bottou2018:book,Gower2019:sgd}.
We will discuss a few of these in Section~\ref{sec:examples} below.
In contrast to the bounded gradient assumption, these assumption allow to chose $\sigma^2=0$ in non-trivial situations and thus allow to recover faster rates in general. However, adapting the proof technique from~\cite{Lacoste2012:simpler} to the relaxed assumptions considered here (cf.\ Lemma~\ref{lem:lacoste} below, or~\cite{Stich2018:sparsified,Grimmer2019:sgd}) gives $f(\bar \xx_T)-f^\star = \cO \bigl( \frac{L^2 R^2 }{\mu T^2} + \frac{\sigma^2}{\mu T}\bigr)$, where the dependence on the initial distance $R$ is not optimal, i.e.\ not exponentially decreasing as in~\cite{Moulines2011:nonasymptotic}.
\citet{Gower2019:sgd} generalize the results of~\cite{Needell2016:sgd} for the convergence of the distance $\norm{\xx_T-\xx^\star}^2$ to the setting considered here and obtain the same rate as stated earlier in this subsection.

To keep our focus, we do not discuss obvious generalizations of our bounds to other settings here. For instance convergence under average smoothness or importance sampling~\cite{Moulines2011:nonasymptotic,Needell2016:sgd} or expected smoothness conditions~\cite{Gower2018:jac}.

\section{Motivating Examples} \label{sec:examples}
In this section we give a few examples that motivate Assumption~\ref{ass:2}.

\paragraph{Example 1 (Gradient Descent).} In the non-stochastic setting, we have $\gg_t = \nabla f(\xx_t)$, $\forall t \geq 0$. If $f$ is $L$-smooth, then $f$ is also $(L,0)$-smooth, as seen by the choice $\yy = \xx^\star$ in the following inequality that holds for convex $L$-smooth functions~\citep[Theorem 2.1.5]{Nesterov2004:book}:
\begin{align}
 \frac{1}{2L} \norm{\nabla f(\xx)- \nabla f(\yy)}^2 \leq f(\xx) - f(\yy) - \lin{\nabla f(\yy),\xx-\yy}\,, \qquad \forall \xx,\yy \in \R^n\,. \label{eq:smoothness}
\end{align}
Hence, we recover the $\cO\bigl(L \norm{\xx_0 - \xx^\star}^2 \exp\bigl[- \frac{\mu T}{L} \bigr]\bigr)$ convergence rate which coincides with the best known rate~\citep[Theorem 2.1.15]{Nesterov2004:book} for the function value convergence in this setting.\footnote{The constant $L$ is tight here%
~\cite{Nesterov2004:book}.
However---as a side remark---we like to point out that our proof reveals the improved bound $\E {f(\bar \xx_T)} -f^\star + \mu \E {\norm{\xx_{T+1}-\xx^\star}^2} = \cO \left(\mu \norm{\xx_0 - \xx^\star}^2 \exp\bigl[-\frac{\mu T}{L}\bigr] + \frac{\sigma^2}{\mu T} \right)$ \emph{if} $T = \Omega\bigl( \frac{\mu}{L} \bigr)$ is sufficiently large (an assumption that appears sometimes in the literature---though does not improve the worst-case complexity for arbitrary $T$). }

\paragraph{Example 2 (Stochastic Gradient Descent).} In the stochastic setting, we have $\gg_t = \nabla f(\xx_t) + \boldsymbol{\xi}_t$, where $\{\boldsymbol{\xi}_t\}_{t \geq 0}$  are independent, zero-mean noise terms, with uniformly bounded second moment $\E{\norm{\boldsymbol{\xi}_t}^2} \leq \sigma^2$ for a constant $\sigma^2 \geq 0$. Again, by relying on~\eqref{eq:smoothness}, we see that Assumption~\ref{ass:2} is satisfied for $L$-smooth functions:
\begin{align*}
 \Eb{\norm{\gg_t}^2 \mid \cF_{t}} = \norm{\nabla f(\xx_t)}^2 + \Eb{\norm{\boldsymbol{\xi}_t}^2 \mid \cF_t}  \leq (2L(f(\xx_t)-f^\star) + \sigma^2\,.
\end{align*}
Hence, we recover the $\cO \bigl( \frac{\sigma^2}{\mu T} \bigr)$ convergence rate of SGD for the function values and the  $\cO \bigl( \frac{\sigma^2}{\mu^2 T} \bigr)$ rate for the last iterate---which are the best known rates~\cite{Rakhlin2012suffix,Shamir2013:averaging}. We like to point out that we here do not need to rely on the frequently used bounded-gradient assumption, as e.g.\ in~\cite{Lacoste2012:simpler,Rakhlin2012suffix,Shamir2013:averaging}.

SGD has also been analyzed under various similar growths conditions, for instance assumptions of the form $\E{\bigl[\norm{\gg_t}^2 \mid \cF_t \bigr]} \leq \nu_1 + \nu_2 \norm{\nabla f(\xx_t)}^2$, for two constants $\nu_1,\nu_2 \geq 0$, see e.g.~\cite{Schmidt2013:fastconvergence,Bertsekas1996growth,Bottou2018:book,Nguyen2018:bounded}. By virtue of~\eqref{eq:smoothness}, we see that these settings are also comprised in Assumption~\ref{ass:2} and covered here.

\paragraph{Example 3 (Empirical Risk Minimization).} 
In machine learning applications the objective function has often a known sum structure, $f(\xx):=\frac{1}{m}\sum_{i=1}^m f_i(\xx)$ for $f_i \colon \R^n \to \R$ convex and $L$-smooth. By picking one index $i \sim_{\rm u.a.r.} [m]$, uniformly at random, $\gg_t := \nabla f_i(\xx_t)$ is an unbiased, $(2L,\tfrac{2}{m} \sum_{i=1}^m \norm{\nabla f_i(\xx^\star)}^2)$-smooth gradient oracle , as can be seen from (cf.\ \cite{Needell2016:sgd}):
\begin{align*}
 \Eb{ \norm{\gg_t}^2 \mid  \cF_t } &= \EE{i}{ \norm{\nabla f_i(\xx_t) - \nabla f_i(\xx^\star) + \nabla f_i(\xx^\star)}^2 } \\ &\leq 2 \EE{i}{ \norm{\nabla f_i(\xx_t) - \nabla f_i(\xx^\star)}^2} + 2 \EE{i}{ \norm{\nabla f_i(\xx^\star)}^2} \\
 &\stackrel{\eqref{eq:smoothness}}{\leq} 4L(f(\xx_t)- f^\star) + 2 \EE{i}{ \norm{\nabla f_i(\xx^\star)}^2}\,.
\end{align*}
When the loss at the optimum vanishes, i.e.\ $\nabla f_i(\xx^\star) = 0$, $\forall i \in [m]$---the so called \emph{interpolation setting}---we have $\sigma^2=0$ and we recover linear convergence of SGD, as e.g.\ in~\cite{Schmidt2013:fastconvergence,Needell2016:sgd,Ma2018:interpolation}.

The above observation also holds in more general settings, such as e.g.\ under expected smoothness or weak growth conditions (cf.~\cite{Gower2018:jac,Gower2019:sgd}).

\section{Convergence Analysis Part I---Deriving a Recursion}
Following standard techniques, we prove the following lemma:
\begin{lemma}
\label{lemma:1}
For $\xx_0 \in \R^d$, let $\{\xx_t\}_{t \geq 0}$ denote the iterates of SGD~\eqref{eq:sgd} generated on a function $f$ under Assumptions~\ref{ass:1}--\ref{ass:3} for stepsizes $\gamma_t \leq \frac{1}{2L}$, $\forall t \geq 0$. Then
\begin{align}
 \E{\norm{\xx_{t+1}-\xx^\star}^2} \leq (1-\mu \gamma_t) \E{\norm{\xx_t - \xx^\star}^2} - \gamma_t (\E{f(\xx_t)-f^\star)} + \gamma_t^2 \sigma^2\,. \label{eq:lemma}
\end{align}
\end{lemma}
\begin{proof}
By definition,
\begin{align*}
 \Eb{\norm{\xx_{t+1}-\xx^\star}^2\mid \cF_{t}} &=
 \Eb{\norm{\xx_{t}-\xx^\star}^2 - 2\gamma_t \lin{\gg_t,\xx_t-\xx^\star} + \gamma_t^2 \norm{\gg_t}^2 \mid \cF_t} \\
 &\stackrel{\eqref{def:unbiased}}{=}\norm{\xx_{t}-\xx^\star}^2 - 2 \gamma_t \lin{\nabla f(\xx_t),\xx_t-\xx^\star} + \gamma_t^2 \Eb{\norm{\gg_t}^2 \mid \cF_t} \\
 &\stackrel{(\ref{def:smooth}),(\ref{def:strong})}{\leq} \norm{\xx_{t}-\xx^\star}^2 - 2 \gamma_t \left(\frac{\mu}{2} \norm{\xx_t-\xx^\star}^2 + f(\xx_t)- f^\star \right) + \gamma_t^2 \bigl(2L(f(\xx_t)-f^\star) +\sigma^2\bigr) \,,
\end{align*}
where we also used $\mu$-convexity in the last inequality. By re-arranging and taking expectation on both sides, we get:
\begin{align*}
 \E {\norm{\xx_{t+1}-\xx^\star}^2} & \leq (1-\mu \gamma_t) \E {\norm{\xx_{t}-\xx^\star}^2} - 2 \gamma_t (1- L\gamma_t) (\E{f(\xx_t)} - f^\star) + \gamma_t^2 \sigma^2 \,,
\end{align*}
and the claim follows by observing $(1-L\gamma_t) \geq \frac{1}{2}$ for $\gamma_t \leq \frac{1}{2L}$.
\end{proof}

\paragraph{A classic convergence result (cf.\ Section~\ref{sec:related}).} Lemma~\ref{lem:first} is a key tool that allows to derive convergence results for SGD. To exemplify, we show there a first result for $\mu$-convex functions with $\mu >0$. By choosing constant  stepsizes $\gamma_t \equiv \gamma \leq \frac{1}{2L}$ (to be specified below) and relaxing~\eqref{eq:lemma} to $\E{\norm{\xx_{t+1}-\xx^\star}^2} \leq (1-\mu \gamma) \E{\norm{\xx_t - \xx^\star}^2} + \gamma^2 \sigma^2$ we obtain after unrolling the recurrence,
\begin{align}
\E{\norm{\xx_{T+1}-\xx^\star}^2} \leq (1-\mu \gamma)^T \norm{\xx_0 - \xx^\star}^2 + \frac{\gamma \sigma^2}{\mu}\,. \label{eq:sgd_norm}
\end{align} This intermediate results shows that SGD with constant stepsizes reduces the initial error term $\norm{\xx_0 - \xx^\star}^2$ linearly, but only converges towards a $\frac{\gamma \sigma^2}{\mu}$-neighborhood of $\xx^\star$ (cf.\ discussions in e.g.\ \cite{Moulines2011:nonasymptotic,Bottou2018:book}).
To obtain a convergence guarantee that holds for arbitrary accuracy, we need to choose the stepsize $\gamma$ carefully:
\begin{itemize}
 \item If $\frac{1}{2L} \geq \frac{\ln(\max\{2,\mu^2 \norm{\xx_0-\xx^\star}^2 T /\sigma^2\})}{\mu T}$ then we choose $\gamma = \frac{\ln(\max\{2,\mu^2 \norm{\xx_0-\xx^\star}^2 T /\sigma^2\})}{\mu T}$.
 \item If otherwise $\frac{1}{2L} < \frac{\ln(\max\{2,\mu^2 \norm{\xx_0-\xx^\star}^2 T /\sigma^2 \})}{\mu T}$ then we pick $\gamma = \frac{1}{2L}$.
\end{itemize}
With these choices of $\gamma$, we can show\footnote{We refer the readers to the proof of Lemma~\ref{lemma:1} below for detailed computations in a very similar setting.}
 \begin{align*}
 \E{\norm{\xx_{T+1}-\xx^\star}^2}  =  \tilde \cO \left(  \norm{\xx_0-\xx^\star}^2 \exp \left[- \frac{\mu T}{2L} \right] + \frac{\sigma^2}{\mu^2 T} \right) \,.
 \end{align*}
This result does not show convergence of the function values $f(\xx_t)-f^\star$ and the $\tilde \cO(\cdot)$ notation hides logarithmic factors. In the next section, we show how we can address both these issues.

\section{Convergence Analysis Part II---Solving the Recursion}
In this section, we consider two non-negative sequences $\{r_t\}_{t\geq 0}$, $\{s_t\}_{t\geq 0}$, that satisfy the relation
\begin{align}
\label{eq:rec}
 r_{t+1} \leq (1-a\gamma_t ) r_t - b\gamma_t s_t +  c \gamma_t^2  \,,
\end{align}
for all $t \geq 0$ and 
for parameters $b > 0$, $a, c \geq 0$ and non-negative stepsizes $\{\gamma_t\}_{t \geq 0}$ with $\gamma_t \leq \frac{1}{d}$, $\forall t \geq 0$, for a parameter $d \geq a$, $d > 0$.

By considering the special case $r_t = \norm{\xx_{t}-\xx^\star}^2$, $s_t = (\E{f(\xx_t)}-f^\star)$, $a=\mu$, $b=1$, $c=\sigma^2$ and $d=2L$, we see that~\eqref{eq:rec} comprises the setting of Lemma~\ref{lemma:1}, and thus the three lemmas that follow below will prove the claims from Section~\ref{sec:contribution}.

\paragraph{Constant Stepsizes (with Log Terms).}
First, we derive a suboptimal (up to polylogarithmic factors) solution of~\eqref{eq:rec}. 
\begin{lemma}
\label{lemma:1}
Let  $\{r_t\}_{t\geq 0}$, $\{s_t\}_{t\geq 0}$ as in~\eqref{eq:rec} and $a > 0$. 
Then there exists a constant stepsize $ \gamma_t \equiv \gamma \leq \frac{1}{d}$
such that for weights $w_t := (1-a\gamma)^{-(t+1)}$ and $W_T := \sum_{t=0}^T w_t$ it holds:
\begin{align*}
 \frac{b}{W_T} \sum_{t=0}^{T}s_t w_t + a r_{T+1} = \tilde \cO \left(d r_0 \exp\left[-\frac{aT}{d}\right] + \frac{c}{a T} \right)\,.
\end{align*}
\end{lemma}

\paragraph{Decreasing Stepsizes (Avoiding Log Terms).}
In Lemma~\ref{lemma:1} above we collected suboptimal logarithmic terms.
The averaging scheme with exponentially decreasing weights has a too short effective window to reduce the variance at the optimal $\cO\bigl(\frac{1}{T}\bigr)$ rate.
In contrast, averaging schemes with polynomial weights can in general achieve the optimal $\cO\bigl(\frac{1}{T}\bigr)$ decrease of the statistical term, but do not decrease the optimization term exponentially fast (see e.g.~\cite{Lacoste2012:simpler},\cite{Shamir2013:averaging}).
This suggests that a combination of these averaging strategies might yield the best results.
We analyze a simple two-phase scheme, that first performs $\frac{T}{2}$ iterations without averaging and then switches to suffix averaging scheme for the remaining iterations (this analysis could be generalized to $\alpha$-suffix averaging as in~\cite{Rakhlin2012suffix}). 

\begin{lemma} \label{lemma:2}
Let  $\{r_t\}_{t\geq 0}$, $\{s_t\}_{t\geq 0}$ as in~\eqref{eq:rec}, $a > 0$. Then there exists stepsizes $\gamma_t \leq \frac{1}{d}$ and weighs $w_t \geq 0$, $W_T := \sum_{t=0}^T w_t$, such that:
\begin{align*}
 \frac{b}{W_T} \sum_{t=0}^{T}s_t w_t + a r_{T+1} \leq   32 d r_0 \exp \left[-\frac{a T}{2d} \right] + \frac{36c}{aT}\,.
\end{align*}
\end{lemma}

\paragraph{Sublinear Rate ($a=0$).} The previous lemma allows to derive the main result presented in this note. For completeness, we also recite a lemma that solves the recursion in the special case $a=0$.
\begin{lemma}[\cite{Kgoogle:cofefe}]
\label{lemma:3} 
Let $\{r_t\}_{t\geq 0}$, $\{s_t\}_{t\geq 0}$ as in~\eqref{eq:rec} for $a \geq 0$. 
Then there exists a constant stepsize $ \gamma_t \equiv \gamma \leq \frac{1}{d}$
such that
\begin{align*}
 \frac{b}{T+1} \sum_{t=0}^T s_t  \leq \frac{d r_0}{T+1} + \frac{2\sqrt{c r_0}}{\sqrt{T+1}}\,.
\end{align*}
\end{lemma}

\paragraph{SGD convergence rates.}
To conclude this section, we now briefly summarize our main result that follows by replacing the variables in Lemmas~\ref{lemma:1}--\ref{lemma:3} by the values stated at the beginning of this section, and observing $f(\bar \xx_T) \leq \frac{1}{W_T}\sum_{t=0}^T w_t f(\xx_t)$ for convex $f$.
\begin{theorem}
For $\xx_0 \in \R^d$, let $\{\xx_t\}_{t \geq 0}$ denote the iterates of SGD~\eqref{eq:sgd} generated on a function $f$ under Assumptions~\ref{ass:1}--\ref{ass:3} for stepsizes $\gamma_t \leq \frac{1}{2L}$, $\forall t \geq 0$. Then there exists stepsizes $\gamma_t \leq \frac{1}{2L}$ and weights $w_t\geq 0$ such that it holds for all $T \geq 0$:
\begin{align*}
 \E{ f(\bar \xx_T) - f^\star} + \mu \E{\norm{\xx_{T+1}-\xx^\star}^2} &\leq \min \left\{ 64 L R^2 \exp \left[-\frac{\mu T}{4L} \right] + \frac{36 \sigma^2}{\mu T}\, , \,\, \frac{2LR^2}{T} + \frac{2 \sigma R}{\sqrt{T}} \, \right\} \,,
\end{align*}
where here  $R := \norm{\xx_0- \xx^\star}$, and again $W_T := \sum_{t=0}^T w_t$ and $\bar \xx_T := \sum_{t=0}^T w_t \xx_t$.
\end{theorem}
\begin{proof}
The theorem follows from the 
decrease Lemma~\ref{lem:first} and
Lemmas~\ref{lemma:2} and~\ref{lemma:3} with $a=\mu$, $b=1$, $c=\sigma^2$ and $d=2L$. For this we observe that all stepsizes $\gamma_t \leq \frac{1}{d} \equiv \frac{1}{2L}$, $\forall t \geq 0$.
\end{proof}
\noindent We here did not explicitly state the (suboptimal) convergence result for tuned constant stepsizes that follows directly from Lemma~\ref{lemma:1}.

\section{Discussion}

We study the iteration complexity of the (stochastic) gradient descent method and recover---simultaneously---the best known rates for the function value suboptmality for an average iterate of SGD and the distance to the optimal solution of the last iterate of SGD. Our analysis focuses on the general stochastic setting, but---as a special case---we also recover the exponential convergence rates in the deterministic setting. This unified analysis address several shortcomings of previous works. 

Whilst we only consider (strongly) convex functions here, further extension of the framework to larger function classes would obviously be an interesting future direction. We would like to remark that Assumption~\ref{ass:2} potentially also covers a much larger class of functions than the few examples discussed in Section~\ref{sec:examples}. For instance, the \emph{approximate gradient oracles} introduced in~\cite{Devolder2014:inexact} satisfy this assumption as well (cf.~\citep[Theorem 1]{Devolder2014:inexact}) and, interestingly, H\"{o}lder continuous functions (which are in general not continuously differentiable) still admit approximate gradient oracles. However, as these oracles are not unbiased in general, Assumption~\ref{ass:1} prevents the immediate application our framework in this extended setting (though, extension of our results under mild relaxations of the unbiasedness Assumption~\ref{ass:1}, similar as e.g. in~\cite{Bottou2018:book}, are immediately possible).

A small drawback our results is that one needs knowledge of $T$ and the problem parameters $\mu,L,\sigma$ to implement the schemes presented here (to decide on the stepsize, and for switching to the suffix averaging). In practice, some of these limitations can be remedied by the doubling trick. Further, just knowing $T$ up to some constant factor approximation is sufficient to recover the optimal rate up to constant factors.

\section*{Acknowledgments}
We would like to thank Martin Jaggi for his support and his helpful comments and Praneeth Karimireddy for his suggestions to improve the first version of this manuscript.

\bibliography{simpler_proof}

\begin{thebibliography}{23}
\providecommand{\natexlab}[1]{#1}
\providecommand{\url}[1]{\texttt{#1}}
\expandafter\ifx\csname urlstyle\endcsname\relax
  \providecommand{\doi}[1]{doi: #1}\else
  \providecommand{\doi}{doi: \begingroup \urlstyle{rm}\Url}\fi

\bibitem[Bach and Moulines(2011)]{Moulines2011:nonasymptotic}
Francis~R. Bach and Eric Moulines.
\newblock
  \href{http://papers.nips.cc/paper/4316-non-asymptotic-analysis-of-stochastic-approximation-algorithms-for-machine-learning.pdf}{Non-asymptotic
  analysis of stochastic approximation algorithms for machine learning}.
\newblock In \emph{Advances in Neural Information Processing Systems 24}, pages
  451--459. Curran Associates, Inc., 2011.

\bibitem[Bertsekas and Tsitsiklis(1996)]{Bertsekas1996growth}
Dimitri~P. Bertsekas and John~N. Tsitsiklis.
\newblock \emph{Neuro-Dynamic Programming}.
\newblock Athena Scientific, 1st edition, 1996.
\newblock ISBN 1886529108.

\bibitem[Bottou et~al.(2018)Bottou, Curtis, and Nocedal]{Bottou2018:book}
L.~Bottou, F.~Curtis, and J.~Nocedal.
\newblock \href{https://doi.org/10.1137/16M1080173}{Optimization methods for
  large-scale machine learning}.
\newblock \emph{SIAM Review}, 60\penalty0 (2):\penalty0 223--311, 2018.
\newblock \doi{10.1137/16M1080173}.

\bibitem[Chung(1954)]{Chung1954:sgd}
K.~L. Chung.
\newblock \href{https://doi.org/10.1214/aoms/1177728716}{On a stochastic
  approximation method}.
\newblock \emph{Ann. Math. Statist.}, 25\penalty0 (3):\penalty0 463--483, 09
  1954.
\newblock \doi{10.1214/aoms/1177728716}.

\bibitem[Devolder et~al.(2014)Devolder, Glineur, and
  Nesterov]{Devolder2014:inexact}
Olivier Devolder, Fran{\c{c}}ois Glineur, and Yurii Nesterov.
\newblock \href{https://doi.org/10.1007/s10107-013-0677-5}{First-order methods
  of smooth convex optimization with inexact oracle}.
\newblock \emph{Mathematical Programming}, 146\penalty0 (1):\penalty0 37--75,
  Aug 2014.
\newblock ISSN 1436-4646.
\newblock \doi{10.1007/s10107-013-0677-5}.

\bibitem[Gower et~al.(2018)Gower, Richt\'{a}rik, and Bach]{Gower2018:jac}
Robert~M. Gower, Peter Richt\'{a}rik, and Francis Bach.
\newblock \href{https://arxiv.org/abs/1805.02632}{Stochastic quasi-gradient
  methods: Variance reduction via jacobian sketching}.
\newblock \emph{CoRR}, abs/1805.02632, 2018.

\bibitem[Gower et~al.(2019)Gower, Loizou, Qian, Sailanbayev, Shulgin, and
  Richt{\'a}rik]{Gower2019:sgd}
Robert~M. Gower, Nicolas Loizou, Xun Qian, Alibek Sailanbayev, Egor Shulgin,
  and Peter Richt{\'a}rik.
\newblock \href{http://proceedings.mlr.press/v97/qian19b.html}{{SGD}: General
  analysis and improved rates}.
\newblock In \emph{Proceedings of the 36th International Conference on Machine
  Learning}, volume~97 of \emph{Proceedings of Machine Learning Research},
  pages 5200--5209. PMLR, 2019.

\bibitem[Grimmer(2019)]{Grimmer2019:sgd}
B.~Grimmer.
\newblock \href{https://doi.org/10.1137/18M117306X}{Convergence rates for
  deterministic and stochastic subgradient methods without lipschitz
  continuity}.
\newblock \emph{SIAM Journal on Optimization}, 29\penalty0 (2):\penalty0
  1350--1365, 2019.
\newblock \doi{10.1137/18M117306X}.
\newblock arXiv:1712.04104.

\bibitem[Karimireddy et~al.(2019)Karimireddy, Kale, Mohri, Reddi, Stich, and
  Suresh]{Kgoogle:cofefe}
Sai~P. Karimireddy, Satyen Kale, Mehryar Mohri, Sashank~J. Reddi, Sebastian~U.
  Stich, and Ananda~T. Suresh.
\newblock \href{https://arxiv.org/abs/1910.06378}{{SCAFFOLD}: Stochastic
  controlled averaging for on-device federated learning}.
\newblock \emph{CoRR}, abs/1910.06378, 2019.

\bibitem[Lacoste-Julien et~al.(2012)Lacoste-Julien, Schmidt, and
  Bach]{Lacoste2012:simpler}
Simon Lacoste-Julien, Mark~W. Schmidt, and Francis~R. Bach.
\newblock \href{https://arxiv.org/abs/1212.2002}{A simpler approach to
  obtaining an {$O(1/t)$} convergence rate for the projected stochastic
  subgradient method}.
\newblock \emph{CoRR}, abs/1212.2002, 2012.

\bibitem[Ma et~al.(2018)Ma, Bassily, and Belkin]{Ma2018:interpolation}
Siyuan Ma, Raef Bassily, and Mikhail Belkin.
\newblock \href{http://proceedings.mlr.press/v80/ma18a.html}{The power of
  interpolation: Understanding the effectiveness of {SGD} in modern
  over-parametrized learning}.
\newblock In \emph{Proceedings of the 35th International Conference on Machine
  Learning}, volume~80 of \emph{Proceedings of Machine Learning Research},
  pages 3325--3334. PMLR, 2018.

\bibitem[Necoara et~al.(2019)Necoara, Nesterov, and
  Glineur]{Necoara2019:linear}
I.~Necoara, Yu. Nesterov, and F.~Glineur.
\newblock \href{https://doi.org/10.1007/s10107-018-1232-1}{Linear convergence
  of first order methods for non-strongly convex optimization}.
\newblock \emph{Mathematical Programming}, 175\penalty0 (1):\penalty0 69--107,
  May 2019.
\newblock ISSN 1436-4646.

\bibitem[Needell et~al.(2016)Needell, Srebro, and Ward]{Needell2016:sgd}
Deanna Needell, Nathan Srebro, and Rachel Ward.
\newblock \href{https://doi.org/10.1007/s10107-015-0864-7}{Stochastic gradient
  descent, weighted sampling, and the randomized {Kaczmarz} algorithm}.
\newblock \emph{Mathematical Programming}, 155\penalty0 (1):\penalty0 549--573,
  Jan 2016.
\newblock ISSN 1436-4646.
\newblock \doi{10.1007/s10107-015-0864-7}.

\bibitem[Nemirovski et~al.(2009)Nemirovski, Juditsky, Lan, and
  Shapiro]{Nemirovski2009:last}
A.~Nemirovski, A.~Juditsky, G.~Lan, and A.~Shapiro.
\newblock \href{https://doi.org/10.1137/070704277}{Robust stochastic
  approximation approach to stochastic programming}.
\newblock \emph{SIAM Journal on Optimization}, 19\penalty0 (4):\penalty0
  1574--1609, 2009.
\newblock \doi{10.1137/070704277}.

\bibitem[Nesterov(2004)]{Nesterov2004:book}
Yurii Nesterov.
\newblock \emph{Introductory Lectures on Convex Optimization}, volume~87 of
  \emph{Springer Science \& Business Media}.
\newblock Springer US, Boston, MA, 2004.

\bibitem[Nguyen et~al.(2018)Nguyen, Nguyen, van Dijk, Richt{\'a}rik,
  Scheinberg, and Tak{\'a}{\v c}]{Nguyen2018:bounded}
Lam Nguyen, Phuong~Ha Nguyen, Marten van Dijk, Peter Richt{\'a}rik, Katya
  Scheinberg, and Martin Tak{\'a}{\v c}.
\newblock \href{http://proceedings.mlr.press/v80/nguyen18c.html}{{SGD} and
  {Hogwild}! {C}onvergence without the bounded gradients assumption}.
\newblock In \emph{Proceedings of the 35th International Conference on Machine
  Learning}, volume~80 of \emph{Proceedings of Machine Learning Research},
  pages 3750--3758. PMLR, 2018.

\bibitem[Polyak(1990)]{Polyak1990:averaging}
B.~T. Polyak.
\newblock New method of stochastic approximation type.
\newblock \emph{Autom. Remote Control}, 51\penalty0 (7):\penalty0 937--946,
  1990.

\bibitem[Rakhlin et~al.(2012)Rakhlin, Shamir, and Sridharan]{Rakhlin2012suffix}
Alexander Rakhlin, Ohad Shamir, and Karthik Sridharan.
\newblock \href{http://dl.acm.org/citation.cfm?id=3042573.3042774}{Making
  gradient descent optimal for strongly convex stochastic optimization}.
\newblock In \emph{Proceedings of the 29th International Coference on
  International Conference on Machine Learning}, pages 1571--1578, 2012.
\newblock ISBN 978-1-4503-1285-1.

\bibitem[Robbins and Monro(1951)]{Robbins:1951sgd}
Herbert Robbins and Sutton Monro.
\newblock A stochastic approximation method.
\newblock \emph{The Annals of Mathematical Statistics}, 22\penalty0
  (3):\penalty0 400--407, September 1951.

\bibitem[Ruppert(1988)]{Ruppert1988:averaging}
D.~Ruppert.
\newblock Efficient estimations from a slowly convergent {Robbins-Monro}
  process.
\newblock Technical Report 781, 1988.

\bibitem[Schmidt and Roux(2013)]{Schmidt2013:fastconvergence}
Mark Schmidt and Nicolas~Le Roux.
\newblock \href{https://arxiv.org/abs/1308.6370}{Fast convergence of stochastic
  gradient descent under a strong growth condition}.
\newblock \emph{CoRR}, abs/1308.6370, 2013.

\bibitem[Shamir and Zhang(2013)]{Shamir2013:averaging}
Ohad Shamir and Tong Zhang.
\newblock \href{http://proceedings.mlr.press/v28/shamir13.html}{Stochastic
  gradient descent for non-smooth optimization: Convergence results and optimal
  averaging schemes}.
\newblock In \emph{Proceedings of the 30th International Conference on Machine
  Learning}, volume~28 of \emph{Proceedings of Machine Learning Research},
  pages 71--79. PMLR, 2013.

\bibitem[Stich et~al.(2018)Stich, Cordonnier, and Jaggi]{Stich2018:sparsified}
Sebastian~U Stich, Jean-Baptiste Cordonnier, and Martin Jaggi.
\newblock
  \href{http://papers.nips.cc/paper/7697-sparsified-sgd-with-memory.pdf}{Sparsified
  {SGD} with memory}.
\newblock In \emph{Advances in Neural Information Processing Systems 31}, pages
  4447--4458. Curran Associates, Inc., 2018.

\end{thebibliography}
\bibliographystyle{myplainnat}

\newpage
\onecolumn
\appendix

\section{Technical Lemmas}

\subsection{Constant Stepsizes (with Log Terms)}

\begin{proof}[Proof of Lemma~\ref{lemma:1}]
We start by re-arranging~\eqref{eq:rec} and multiplying both sides with $w_t$:
\begin{align*}
 b s_t w_t \leq \frac{w_t (1-a\gamma) r_t}{\gamma} - \frac{w_t r_{t+1}}{\gamma}  + c \gamma w_t = \frac{w_{t-1} r_t}{\gamma} - \frac{w_t r_{t+1}}{\gamma}  + c \gamma w_t \,.
\end{align*} 
By summing from $t=0$ to $t=T$, we obtain a telescoping sum:
\begin{align*}
 \frac{b}{W_T} \sum_{t=0}^{T} s_t w_t \leq \frac{1}{\gamma W_T} \left(w_0 (1-a\gamma)r_0 - w_{T} r_{T+1}\right) + c \gamma\,, 
\end{align*}
and hence
\begin{align*}
 \frac{b}{W_T} \sum_{t=0}^{T}s_t w_t + \frac{w_{T} r_{T+1}}{\gamma W_T} \leq \frac{r_0}{\gamma W_{T}} + c \gamma\,.
\end{align*}
With the estimates
\begin{itemize}
 \item $W_T = (1-a\gamma)^{-(T+1)}\sum_{t=0}^T (1-a \gamma)^t \leq \frac{w_{T}}{a \gamma}$ (here we leverage $a\gamma \leq \frac{a}{d} \leq 1$),
 \item and $W_T \geq w_{T} = (1-a\gamma)^{-(T+1)}$,
\end{itemize} 
we can further simplify the left and right hand sides:
\begin{align}
\frac{b}{W_T} \sum_{t=0}^{T}s_t w_t + a r_{T+1} \leq (1-a\gamma)^{(T+1)} \frac{r_0}{\gamma} + c \gamma \leq \frac{r_0}{\gamma} \exp \left[-a\gamma (T+1)\right] +c\gamma\,. \label{eq:main}
\end{align}
Now the lemma follows by carefully tuning $\gamma$. Consider the two cases:

\begin{itemize}
 \item If $\frac{1}{d} \geq \frac{\ln(\max\{2,a^2 r_0 T^2 /c\})}{a T}$ then we choose $\gamma = \frac{\ln(\max\{2,a^2 r_0 T^2 /c\})}{a T}$ and get that Equation~\eqref{eq:main} is
 \begin{align*}
  \tilde \cO \left( a r_0 T \exp[-\ln(\max\{2,a^2 r_0 T^2 /c\})  ] \right)  + \tilde \cO \left( \frac{c}{aT} \right) = \tilde \cO \left( \frac{c}{aT} \right)\,,
\end{align*}  
as in case $2 \geq a^2 r_0 T^2/c$ it holds $a r_0 T \leq \frac{2c}{aT}$.
 \item If otherwise $\frac{1}{d} < \frac{\ln(\max\{2,a^2 r_0 T^2 /c\})}{a T}$ then we pick $\gamma = \frac{1}{d}$ and get that Equation~\eqref{eq:main} is
 \begin{align*}
  d r_0 \exp \left[- \frac{aT}{d} \right] + \frac{c}{d} \leq  d r_0 \exp \left[- \frac{aT}{d} \right] + \frac{c \ln(\max\{2,a^2 r_0 T^2 /c\})}{aT} = \tilde \cO \left(  d r_0 \exp \left[- \frac{aT}{d} \right] + \frac{c}{aT} \right) \,.  %
 \end{align*}
\end{itemize}
Collecting these two cases concludes the proof.
\end{proof}

\subsection{Decreasing Stepsizes (Avoiding Log Terms)}
For the proof of Lemma~\ref{lemma:2} we need auxiliary results, for both of which we do not claim much novelty here.

\begin{lemma}
\label{lem:first}
Let  $\{r_t\}_{t\geq 0}$, $\{s_t\}_{t\geq 0}$ be as in~\eqref{eq:rec} for $a > 0$ and for constant stepsizes $\gamma_t \equiv \gamma := \frac{1}{d}$, $\forall t \geq 0$. Then it holds for all $T \geq 0$:
\begin{align*}
 r_{T} \leq r_0 \exp\left[- \frac{a T}{d} \right] + \frac{c}{ad}\,.
\end{align*}
\end{lemma}
\begin{proof}
This follows by relaxing~\eqref{eq:rec}, and unrolling:
\begin{align}
 r_{T} &\leq (1-a\gamma)r_{T-1} + c \gamma^2 \leq (1-a\gamma)^T r_0 +  c \gamma^2 \sum_{t=0}^{T-1} (1-a\gamma)^t \leq (1-a\gamma)^T r_0 + \frac{c \gamma}{a}\,. \qedhere 
\end{align}
\end{proof}

\noindent The next lemma is similar to the result derived in~\citep[Sec. 3.2]{Lacoste2012:simpler}, except that we cannot chose the stepsizes $\gamma_t$ arbitrarily large and hence need to take care of the initial conditions. Similar results were presented e.g.\ in~\cite{Stich2018:sparsified,Grimmer2019:sgd}.
\begin{lemma}
\label{lem:lacoste}
Let  $\{r_t\}_{t\geq 0}$, $\{s_t\}_{t\geq 0}$ as in~\eqref{eq:rec} for $a > 0$ and for decreasing stepsizes $\gamma_t := \frac{2}{a(\kappa + t)}$, $\forall t \geq 0$, with parameter $\kappa := \frac{2d}{a}$, and weights $w_t := (\kappa + t)$. Then
\begin{align*}
 \frac{b}{W_T} \sum_{t=0}^{T}s_t w_t + a r_{T+1} \leq \frac{2a \kappa^2 r_0}{T^2} +  \frac{2c}{aT} \,,
\end{align*}
where here again $W_T := \sum_{t=0}^T w_t$.
\end{lemma}
\begin{proof}
We start as in the proof of Lemma~\ref{lemma:1},
\begin{align*}
b s_t w_t &\leq \frac{w_t(1-a\gamma_t)r_t}{\gamma_t} - \frac{w_t r_{t+1}}{\gamma_t} + c \gamma_t w_t \\
&= a (\kappa + t) (\kappa + t-2) r_t - a (\kappa + t)^2 r_{t+1} + \frac{c}{a} \\
&\leq a(\kappa + t-1)^2 r_{t} -  a (\kappa + t)^2 r_{t+1} + \frac{c}{a} \,,
\end{align*} 
where the equality follows from the definition of $\gamma_t$ and $w_t$ and the inequality from $(\kappa + t)(\kappa + t-2) = (\kappa + t-1)^2 - 1 \leq  (\kappa + t-1)^2$. Again we have a telescoping sum:
\begin{align*}
 \frac{b}{W_T} \sum_{t=0}^T s_t w_t + \frac{a(\kappa + T)^2 r_{T+1}}{W_T} \leq \frac{a \kappa^2 r_0}{W_T} + \frac{c (T+1)}{a W_T}\,,
\end{align*}
with 
\begin{itemize}
\item $W_T = \sum_{t=0}^T w_t = \sum_{t=0}^T (\kappa + t) = \frac{(2\kappa+T) (T+1)}{2} \geq \frac{T(T+1)}{2} \geq \frac{T^2}{2}$,
\item and $W_T =  \frac{(2\kappa+T) (T+1)}{2} \leq \frac{2(\kappa + T)(1+T)}{2} \leq (\kappa + T)^2$ for $\kappa = \frac{2d}{a} \geq 1$.
\end{itemize}
By applying these two estimates we conclude the proof.
\end{proof}

\noindent We can now combine the findings of these two lemmas.
\begin{proof}[Proof of Lemma~\ref{lemma:2}]
For integer $T \geq 0$, we choose stepsizes and weights as follows:
\begin{align*}
&\text{if $T \leq \frac{d}{a}$}\,, & \gamma_t &= \frac{1}{d}\,, & w_t &= (1-a\gamma_t)^{-(t+1)} = \left(1-\frac{a}{d}\right)^{-(t+1)}, \\
&\text{if $T > \frac{d}{a}$ and $t < t_0$}, & \gamma_t &= \frac{1}{d}\,, & w_t &= 0\,, \\
&\text{if $T > \frac{d}{a}$ and $t \geq t_0$}, & \gamma_t &= \frac{2}{a(\kappa + t-t_0)}\,,   & w_t &= (\kappa + t-t_0)^2\,, 
\end{align*}
for $\kappa = \frac{2d}{a}$ and $t_0 =  \bigl\lceil \frac{T}{2} \bigr\rceil$. We will now show that these choices imply the claimed result.

\noindent We start with the case $T \leq \frac{d}{a}$. This case is similar to the proof of Lemma~\ref{lemma:1} and it suffices to consider Equation~\ref{eq:main} for the choice $\gamma = \frac{1}{d}$. We observe that Equation~\ref{eq:main} simplifies to
\begin{align*}
 d r_0 \exp \left[-\frac{a T}{d} \right] + \frac{c}{d} \leq d r_0 \exp \left[-\frac{a T}{d} \right] + \frac{c}{aT}\,.
\end{align*}
\noindent If $T > \frac{d}{a}$, then we obtain from Lemma~\ref{lem:first} that
\begin{align*}
 r_{t_0} \leq  r_0 \exp \left[-\frac{a T}{2d} \right] + \frac{c}{ad}  \,.
\end{align*}
From Lemma~\ref{lem:lacoste} we have for the second half of the iterates:
\begin{align*}
\frac{b}{W_T} \sum_{t=0}^{T}s_t w_t + a r_{T+1} &= \frac{b}{W_T} \sum_{t=t_0}^T s_t w_t + a r_{T+1} \leq  \frac{8a \kappa^2 r_{t_0}}{T^2} + \frac{4c}{aT} \,.
\end{align*}
Now we observe that the restart condition $r_{t_0}$ satisfies:
\begin{align*}
 \frac{a \kappa^2 r_{t_0} }{T^2} = \frac{a \kappa^2 r_0 \exp \left( - \frac{aT}{2d} \right)}{T^2} + \frac{\kappa^2 c}{d T^2} \leq 4a r_0 \exp \left[ - \frac{aT}{2d} \right] + \frac{4c}{aT}\,,
\end{align*}
because $T \geq \frac{d}{a}$. These inequalities show the claim.
\end{proof}

\subsection{Sub-linear Convergence rate}

\begin{proof}[Proof of Lemma~\ref{lemma:3}]
We start by re-arranging~\eqref{eq:rec} and summing from $t=0$ to $t=T$:
\begin{align*}
 \frac{b}{T+1} \sum_{t=0}^{T} s_t \leq \frac{1}{T+1}\sum_{t=0}^T \left( \frac{r_t}{\gamma} - \frac{r_{t+1}}{\gamma} \right) + c \gamma  \leq \frac{r_0}{\gamma(T+1)} + c \gamma\,. 
\end{align*}
Now we carefully select $\gamma$.
\begin{itemize}
 \item If $\frac{1}{d^2} \leq \frac{ r_0}{c (T+1)} $, then we pick $\gamma = \frac{1}{d}$ with $\gamma \leq \frac{\sqrt{r_0}}{\sqrt{c (T+1)}}$ and verify
 \begin{align*}
  \frac{r_0}{\gamma(T+1)} + c \gamma \leq \frac{d r_0}{T+1} + \frac{\sqrt{c r_0}}{\sqrt{T+1}} \,.
 \end{align*}
 \item If otherwise $\frac{1}{d^2} > \frac{ r_0}{c (T+1)} $ then we pick $\gamma = \frac{\sqrt{r_0}}{\sqrt{c (T+1)}}$ to obtain
 \begin{align*}
 \frac{r_0}{\gamma(T+1)} + c \gamma  \leq \frac{2 \sqrt{c r_0}}{\sqrt{T+1}}\,. & \qedhere
 \end{align*}
\end{itemize}
\end{proof}

\end{document}